%% file: iclr2027_conference.tex
\documentclass{article} % For LaTeX2e
\usepackage{iclr2027_conference,times}

\input{math_commands.tex}

\definecolor{citecolor}{RGB}{0, 80, 200}
\usepackage[colorlinks=true,citecolor=citecolor,linkcolor=red,urlcolor=citecolor]{hyperref}
\usepackage{url}
\usepackage{graphicx}
\usepackage{booktabs}
\usepackage{colortbl}
\usepackage{multirow}
\usepackage{amsmath}
\usepackage{amssymb}
\usepackage{amsthm}
\usepackage{algorithm}
\usepackage{algpseudocode}
\usepackage{enumitem}
\usepackage{tikz}
\usetikzlibrary{positioning, decorations.pathreplacing, calc}
\usepackage{setspace}
\usepackage{subcaption}
\usepackage{pifont}
\usepackage{wrapfig}
\usepackage{caption}

\newcommand{\methodname}{Ripple-Pivot Search}
\newcommand{\methodshort}{RPS}
\newcommand{\effectname}{ripple effect}

\newtheorem{proposition}{Proposition}

\title{{\methodname}: \\ Active Parallel Decoding for Diffusion Large Language Models}

\author{%
Yushi Ye$^1$ \, Xu Chen$^2$ \, Haoyun Jiang$^1$ \, Jinsong Lan$^2$ \, Haihong Tang$^{2, \dagger}$ \\
\textbf{Xiangtao Li}$^3$ \, \textbf{Mingming Gong}$^4$ \, \textbf{Ivor Tsang}$^{5}$ \, \textbf{Yanfeng Wang}$^6$ \, \textbf{Jiangchao Yao}$^{1, \dagger}$ \\
$^{1}$Cooperative Medianet Innovation Center, Shanghai Jiao Tong University \\ 
$^2$Alibaba Group \, $^3$School of Artificial Intelligence, Jilin University \\
$^4$School of Mathematics and Statistics, The University of Melbourne \\
$^5$A*STAR CFAR and Nanyang Technological University \\
$^{6}$School of Artificial Intelligence, Shanghai Jiao Tong University \\
\texttt{\{stephen-ye, Sunarker\}@sjtu.edu.cn}
}

\iclrfinalcopy % Uncomment for camera-ready version, but NOT for submission.
\begin{document}

\maketitle

\input{sections/abstract}
\input{sections/introduction}
\input{sections/background}
\input{sections/preliminary}
\input{sections/method}

\input{sections/experiments}
\input{sections/conclusion}
\input{sections/acknowledgement}

\bibliography{iclr2027_conference}
\bibliographystyle{iclr2027_conference}

\newpage
\appendix
\input{sections/appendix}

\end{document}

%% file: math_commands.tex
\usepackage{amsmath,amsfonts,bm}

\def\eqref#1{equation~\ref{#1}}
\def\1{\bm{1}}

\DeclareMathAlphabet{\mathsfit}{\encodingdefault}{\sfdefault}{m}{sl}
\SetMathAlphabet{\mathsfit}{bold}{\encodingdefault}{\sfdefault}{bx}{n}

%% file: sections/abstract.tex
\begin{abstract}
    Diffusion Large Language Models (dLLMs) have emerged as a competitive alternative to autoregressive language models, offering the potential for substantially faster inference through parallel decoding. 
    Existing parallel decoding schedulers typically commit positions only after they meet a per-position criterion, overlooking how early commitments may benefit subsequent decoding. 
    We identify a \effectname{} in dLLM decoding: proactively committing a \emph{mid-entropy} pivot position can induce a pronounced reduction in uncertainty across the remaining masked positions. 
    This uncertainty reduction allows subsequent steps to unmask more tokens in parallel, thereby accelerating the overall decoding process. 
    % To exploit the \effectname{} under the ambiguity of mid-entropy predictions, we propose \methodname{} (\methodshort{}), a training-free decoding strategy that targets such pivots for early commitment. 
    % {\methodshort} determines the pivot's token assignment via lookahead-based evaluation of plausible candidates and commits only when resolving the pivot yields a net benefit over leaving it masked.
    To exploit the \effectname{}, we propose \methodname{} (\methodshort{}), a novel training-free decoding method that seeks mid-entropy positions as promising candidate pivots (\emph{where to decode}), and determines their token assignment that yields the greatest downstream benefit via lookahead evaluation (\emph{what to decode}).
    Across 3~dLLMs and 4~reasoning and code-generation benchmarks, \methodshort{} achieves 4--10$\times$ wall-clock speedup over the default decoder while preserving generation quality, and improves accuracy over the previous lookahead baseline by up to 5.49\% while delivering higher throughput in most settings. When integrated with KV caching, \methodshort{} further achieves up to 18$\times$ wall-clock speedup over the default decoder.
\end{abstract}

%% file: sections/introduction.tex
\section{Introduction}
\label{sec:intro}

Diffusion large language models (dLLMs)~\citep{DBLP:conf/nips/NieZYZOHZLWL25,DBLP:conf/acl/ZhuWNZWZLWL26,DBLP:journals/corr/abs-2508-15487,DBLP:conf/acl/ChengBLJLZYTWGC26,DBLP:journals/corr/abs-2512-15745} have gained prominence as a viable alternative to autoregressive language models~\citep{DBLP:conf/nips/BrownMRSKDNSSAA20,DBLP:journals/corr/abs-2505-09388,DBLP:journals/fcsc/ZhaoZLTDHZMZLWDYCCJRLTL26}, offering the potential for faster inference by decoding multiple tokens in parallel. In each denoising step, the model produces predictions at all masked positions simultaneously, and a decoding scheduler selects which positions to unmask and assigns tokens at those positions. In practice, however, committing more positions per step increases the risk of error accumulation~\citep{DBLP:journals/corr/abs-2602-23225}, making the balance between decoding speed and generation quality a central challenge in dLLM inference.

To navigate this speed-quality trade-off, a growing body of work on parallel decoding~\citep{DBLP:journals/corr/abs-2507-18578,DBLP:journals/corr/abs-2510-00294,DBLP:journals/corr/abs-2602-06953,DBLP:journals/corr/abs-2602-22868,DBLP:journals/corr/abs-2602-23996} has focused on designing schedulers that commit multiple positions per step under reliability constraints, thereby amortizing the cost of each forward pass. The scheduler is thus responsible for two decisions: \emph{where} to unmask and \emph{what} token to assign. Most schedulers couple the two, committing each position to its greedy prediction once a per-position criterion is satisfied, such as sufficient confidence~\citep{DBLP:journals/corr/abs-2505-22618,DBLP:conf/nips/NieZYZOHZLWL25}, low predictive entropy~\citep{DBLP:journals/corr/abs-2508-15487}, cross-step stability~\citep{DBLP:conf/nips/KimHJPY25,DBLP:journals/corr/abs-2506-10848}, or bounded cumulative entropy~\citep{DBLP:conf/nips/Ben-HamuGSNK25}. More recent lookahead-based methods~\citep{DBLP:journals/corr/abs-2512-16229,DBLP:journals/corr/abs-2511-21103} further test whether committing an additional position can benefit subsequent decoding. However, these methods employ lookahead mainly to determine whether and where to commit. Once a position is selected, its token is typically fixed to the model’s current top-1 prediction. As a result, the search explores different commitment positions but not alternative token assignments, potentially overlooking effective non-greedy decoding trajectories.
\begin{figure}[t]
    \centering
    
    \begin{tabular}{
        @{}
        b{0.57\linewidth}
        @{\hspace{0.01\linewidth}}
        b{0.40\linewidth}
        @{}
    }
    \begin{subfigure}[b]{\linewidth}
        \centering
        \scriptsize
        \setlength{\tabcolsep}{2.2pt}
        \renewcommand{\arraystretch}{1.08}
    
        \resizebox{\linewidth}{!}{%
        \begin{tabular}[b]{
            @{}
            >{\raggedright\arraybackslash}m{0.42\linewidth}
            >{\raggedright\arraybackslash}m{0.30\linewidth}
            @{\hspace{1.2pt}}
            >{\raggedright\arraybackslash}m{0.28\linewidth}
            @{}
        }
        \toprule
        \textbf{Method family}
            & \textbf{Where to decode}
            & \textbf{What to decode} \\
        \midrule
    
        Confidence-based~\citep{
            DBLP:journals/corr/abs-2505-22618,
            DBLP:conf/nips/NieZYZOHZLWL25}
            & Confidence-qualified\newline positions
            & Top-1 (greedy) \\
    
        Entropy-based~\citep{
            DBLP:journals/corr/abs-2508-15487,
            DBLP:conf/nips/Ben-HamuGSNK25}
            & Entropy-qualified\newline positions
            & Top-1 (greedy) \\
    
        Stability-based~\citep{
            DBLP:conf/nips/KimHJPY25,
            DBLP:journals/corr/abs-2506-10848}
            & Cross-step stable\newline positions
            & Top-1 (greedy) \\
    
        Lookahead-based~\citep{
            DBLP:journals/corr/abs-2512-16229,
            DBLP:journals/corr/abs-2511-21103}
            & Lookahead-selected\newline positions
            & Top-1 (greedy) \\
    
        \textbf{RPS (ours)}
            & \textbf{Mid-entropy\newline pivot}
            & \textbf{Lookahead selected\newline (non-greedy)} \\
    
        \bottomrule
        \end{tabular}
        }%
    \end{subfigure}
    
    &
    % ============================================================
    % Right
    % ============================================================
    \begin{subfigure}[b]{\linewidth}
        \centering
    
        \includegraphics[
            width=\linewidth
        ]{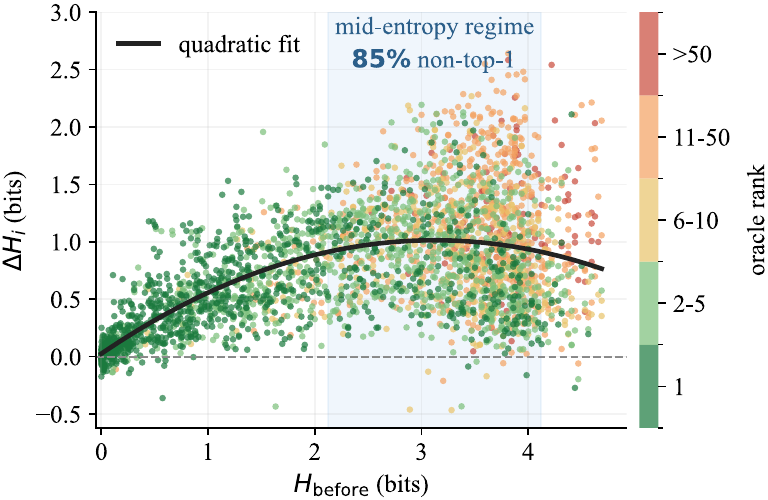}
    \end{subfigure}
    \end{tabular}
    \vspace{-10pt}
    \caption{
    \textbf{Left:} Comparison of parallel decoding schedulers.
    \textbf{Right:} Starting from a partially decoded block, we commit each
    remaining masked position to its oracle token, i.e., the token eventually committed under the standard decoding procedure, and measure the resulting
    entropy reduction $\Delta H_i$ at other masked positions (y-axis) against
    the pre-commit entropy $H_{\mathrm{before}}$ of the resolved position
    (x-axis). Colors indicate the oracle token's rank under the model
    distribution, and the shaded band marks the mid-entropy regime.
    }
    
    \label{fig:ripple}
    \vspace{-15pt}
    \end{figure}

The benefit of early commitment depends critically on which unresolved position is selected. To characterize this, we conduct an oracle analysis over candidate commitments in Fig.~\ref{fig:ripple} (right). This analysis reveals a distinct pattern that we refer to as the \emph{\effectname}: proactively committing a pivot position in the \textbf{mid-entropy regime} induces the strongest downstream uncertainty reduction. Intuitively, such positions are not fully determined, but are already sufficiently tied to the current partial decoding state; resolving them can therefore influence other masked positions more strongly than positions that are either already certain or still weakly constrained. Moreover, the oracle token is not the model's top-1 prediction in 85\% of mid-entropy cases, revealing a mismatch with existing lookahead-based schedulers. As shown in Fig.~\ref{fig:ripple} (left), these methods use lookahead to decide \emph{where} to commit while fixing \emph{what} to the current top-1 prediction. They therefore search over commitment positions but not token assignments, potentially missing beneficial non-greedy trajectories.

Motivated by these findings, we propose \textbf{\methodname{} (\methodshort{})}, a training-free parallel decoding method that exploits beneficial early commitments in the mid-entropy regime. \methodshort{} first applies a two-stage pivot filter to exclude positions that are either overly certain or insufficiently supported by the current predictive distribution, thereby focusing in the mid-entropy regime. 
Given this pivot, 
% \methodshort{} next constructs an adaptive set of plausible token assignments, evaluates their downstream benefit with a single lookahead forward pass, and identifies the most promising candidate for early commitment, so that the selected assignment is the one expected to best amplify the \effectname{} rather than merely the current top-1 prediction. 
\methodshort{} then constructs an adaptive set of plausible token assignments, evaluates their downstream benefits with a single lookahead forward pass, and selects the most promising token for early commitment. This allows \methodshort{} to choose the assignment expected to best amplify the \effectname{}.
% , rather than defaulting to the model’s current top-1 prediction.

Furthermore, to prevent premature commitments from accumulating errors, \methodshort{} commits a token only when it improves over leaving the pivot unmasked. 
% Overall, by jointly modeling where and what to commit, \methodshort{} fully exploits the \effectname{} to accelerate decoding. 
% without compromising generation quality.
In summary, our contributions are:
\begin{itemize}
\item We identify a \emph{\effectname} in dLLM decoding: proactively committing a position in the mid-entropy regime induces the strongest downstream uncertainty reduction, enabling more parallel commits in subsequent steps. Crucially, the oracle token in this regime is frequently not the current top-1 prediction, motivating token assignment beyond greedy decoding.
% \item We propose \methodname{} (\methodshort{}), a training-free decoding method that enables proactive early commitment. \methodshort{} selects a mid-entropy pivot position and determines its token assignment through lookahead-based evaluation of an adaptive set of plausible candidates. To preserve generation quality, \methodshort{} abstains from commitment when no candidate outperforms leaving the pivot unmasked.
\item We propose \methodname{} (\methodshort{}), a novel training-free decoding method that seeks mid-entropy pivots with high potential for downstream uncertainty reduction, and determines its token assignment through lookahead-based evaluation on an adaptive candidate set, committing when the best candidate outperforms leaving the pivot unmasked.
\item We conduct extensive experiments on three dLLMs across four reasoning and code-generation benchmarks. \methodshort{} achieves a \textbf{4--10}$\times$ inference speedup over the default one-token-per-step baseline while preserving generation quality, and improves accuracy over the previous lookahead baseline by up to \textbf{5.49\%} while delivering higher throughput in most settings. When integrated with KV caching, \methodshort{} further achieves up to \textbf{18}$\times$ wall-clock speedup over the default decoder.
\end{itemize}

%% file: sections/background.tex
\section{Related Work}
\label{sec:related}

\paragraph{Criterion-based Parallel Decoding.}
The most direct approach to accelerating dLLM inference commits multiple positions per decoding step based on per-position prediction statistics. Fast-dLLM~\citep{DBLP:journals/corr/abs-2505-22618} unmasks every position whose predictive confidence exceeds a fixed threshold and introduces block-wise KV caching to reduce per-forward cost. KLASS~\citep{DBLP:conf/nips/KimHJPY25} augments confidence with a KL divergence criterion that tracks prediction stability across consecutive steps, unmasking only positions that are both confident and temporally consistent. EB-Sampler~\citep{DBLP:conf/nips/Ben-HamuGSNK25} controls the number of tokens unmasked per step by bounding the cumulative entropy of the newly committed positions. Learn2PD~\citep{DBLP:journals/corr/abs-2509-25188} further replaces fixed heuristics with a lightweight learned filter that predicts whether each current token prediction matches the final output, yielding an adaptive criterion for parallel unmasking.

\paragraph{Lookahead-based Parallel Decoding.}
% Lookahead has been extensively studied for accelerating autoregressive LLMs, where speculative or auxiliary predictions are verified by the base model~\citep{DBLP:conf/icml/LeviathanKM23,DBLP:conf/icml/CaiLGPLCD24,DBLP:journals/corr/abs-2503-01840}. In dLLM decoding, another family adopts this lookahead philosophy: they first hypothesize token assignments and then let the model evaluate or verify the outcome. 
Lookahead-based dLLM decoding accelerates inference by proposing token assignments and evaluating their downstream effects.
WINO~\citep{DBLP:journals/corr/abs-2507-18578} drafts all positions admitted by a relaxed confidence threshold and evaluates the committed tokens under the enriched context, remasking those whose verification confidence falls below a stricter threshold. LoPA~\citep{DBLP:journals/corr/abs-2512-16229} forms lookahead branches from the highest-confidence positions that remain masked after the standard decoding update and selects the branch with the greatest future confidence. From an information-theoretic perspective, ETE~\citep{DBLP:journals/corr/abs-2511-21103} targets low-parallelism steps in which the standard decoding commits only a few tokens. At such steps, it explores high-information positions near a prescribed confidence level and selects the one that unlocks the most downstream high-confidence tokens.
% Both LoPA and ETE use lookahead to refine \emph{where} to commit but evaluate each selected position under its greedy token assignment. In contrast, \methodshort{} targets the mid-entropy regime based on the empirically observed \effectname{}, and uses lookahead to determine not only \emph{where} to intervene but also \emph{what} token to assign by evaluating plausible token candidates.
Yet these methods primarily use lookahead to decide \emph{where} to commit with keeping greedy token assignments. By contrast, \methodshort{} focuses on mid-entropy regime and uses lookahead to jointly decide \emph{where} to commit and \emph{what} token to assign.

% \paragraph{Cache Optimization.}
% A complementary line of work reduces per-forward cost in dLLMs through KV caching. Fast-dLLM~\citep{DBLP:journals/corr/abs-2505-22618} introduces block-wise approximate KV caching that reuses inter-block key-value states across denoising steps, avoiding redundant recomputation for blocks whose content has stabilized. dKV-Cache~\citep{DBLP:journals/corr/abs-2505-15781} caches the KV states of already-decoded tokens and reuses them with a one-step delay, balancing representational freshness against computational savings. d$^2$Cache~\citep{DBLP:journals/corr/abs-2509-23094} applies two-stage token-level selection to determine which KV entries require updating at each step, reducing unnecessary recomputation. These cache techniques are orthogonal to decoding scheduling and can be combined with our method for further acceleration.

%% file: sections/preliminary.tex
\section{Preliminary}
\label{sec:method-prelim}

\paragraph{Notation.} 
% Consider a masked discrete diffusion language model with vocabulary $\mathcal{V}$, which contains a special mask token $\texttt{[MASK]}$. A response of length $L$ is represented as $x \in \mathcal{V}^L$, where each position is either decoded (carrying a token from $\mathcal{V}$) or masked. Let $\mathcal{M} \subseteq \{1, \dots, L\}$ denote the set of currently masked positions. Given $x$, the model returns per-position predictive distributions $\{p_i = p_\theta(\cdot \mid x)\}_{i \in \mathcal{M}}$ over $\mathcal{V}$ in a single forward pass. Write $H(p_i)$ for the predictive entropy at position $i$ and $P_i^{\max} \triangleq \max_{v \in \mathcal{V}} p_i(v)$ for its \emph{confidence}, i.e., its top-$1$ probability.
Consider a masked discrete diffusion language model with vocabulary
$\mathcal{V}$, which contains a special mask token $[\texttt{MASK}]$.
Given a prompt $y$, a response of length $L$ is represented
as $x\in\mathcal{V}^L$, where each position is either decoded or masked.
Let $\mathcal{M}\subseteq\{1,\ldots,L\}$ denote the set of currently masked
response positions. Given the input
$[y\Vert x]$, the model returns predictive
distributions
$\{p_i=p_\theta(\cdot\mid y,x)\}_{i\in\mathcal{M}}$
over $\mathcal{V}$ in a single forward pass. Write $H(p_i)$ for the predictive entropy at position $i$ and $P_i^{\max} \triangleq \max_{v \in \mathcal{V}} p_i(v)$ for its \emph{confidence}, i.e., its top-$1$ probability.

\paragraph{Decoding.} Standard practice~\citep{DBLP:conf/nips/NieZYZOHZLWL25,DBLP:journals/corr/abs-2512-15745,DBLP:journals/corr/abs-2505-22618} adopts a semi-autoregressive schedule: the length-$L$ response is partitioned into contiguous blocks of size $B$ and decoded left-to-right, fully unmasking each block before the next. Within each block, every decoding step shares the same form. Given the current $\{p_i\}_{i \in \mathcal{M}}$, a commit set $\mathcal{S} \subseteq \mathcal{M}$ is selected and each position in $\mathcal{S}$ is assigned its greedy (top-1) prediction,
\begin{equation}
  x_i \;\leftarrow\; \arg\max_{v \in \mathcal{V}} p_i(v) \quad \text{for each } i \in \mathcal{S}. \label{eq:commit}
\end{equation}
Different schedulers are characterized by how they specify $\mathcal{S}$. The three most common rules are:
\begin{align} \label{eq:rules}
    \mathcal{S} = 
    \begin{cases}
    \operatorname{top\text{-}k}_{i \in \mathcal{M}}\big(\left\{P_1^{\max},\dots,P_L^{\max}\right\}\big), & \text{highest-confidence decoding~\citep{DBLP:conf/nips/NieZYZOHZLWL25}}; \\ \\
    \operatorname{top\text{-}k}_{i \in \mathcal{M}}\big(\left\{\!-H(p_1),\dots,\!-H(p_L)\right\}\big), & \text{lowest-entropy decoding~\citep{DBLP:journals/corr/abs-2508-15487}}; \\ \\
    \big\{ i \in \mathcal{M} : P_i^{\max} \geq \tau \big\}, & \text{confidence-aware decoding~\citep{DBLP:journals/corr/abs-2505-22618}}. \\
    \end{cases}
\end{align}
All three rules differ in how $\mathcal{S}$ is constructed but share the same token-assignment mechanism: each committed position receives its greedy prediction (Eq.~\ref{eq:commit}). Positions not selected into $\mathcal{S}$ remain masked until they satisfy the criterion as decoding progresses.
% \begin{itemize}[topsep=2pt,itemsep=4pt,leftmargin=*]
%   \item \textbf{Highest-confidence decoding}~\citep{DBLP:journals/corr/abs-2502-09992}. The commit set contains the top-$k$ positions with the highest confidence, where $k$ follows a predefined schedule:
%     \begin{equation*}
%       \mathcal{S}_{\text{conf}}(k) \;=\; \operatorname{top\text{-}k}_{i \in \mathcal{M}}\big(P_i^{\max}\big).
%     \end{equation*}
%   \item \textbf{Lowest-entropy decoding}~\citep{DBLP:journals/corr/abs-2508-15487}. The commit set contains the top-$k$ positions with the lowest predictive entropy, where $k$ follows a predefined schedule:
%     \begin{equation*}
%       \mathcal{S}_{\text{ent}}(k) \;=\; \operatorname{top\text{-}k}_{i \in \mathcal{M}}\big(\!-H(p_i)\big).
%     \end{equation*}
%   \item \textbf{Confidence-aware decoding}~\citep{DBLP:journals/corr/abs-2505-22618}. The commit set contains every masked position whose confidence exceeds a fixed threshold $\tau$:
%     \begin{equation*}
%       \mathcal{S}_{\text{thr}}(\tau) \;=\; \big\{ i \in \mathcal{M} : P_i^{\max} \geq \tau \big\}.
%     \end{equation*}
% \end{itemize}
% All three rules differ in how $\mathcal{S}$ is constructed but share the same token-assignment mechanism: each committed position receives its greedy prediction (Eq.~\ref{eq:commit}). Positions not selected into $\mathcal{S}$ remain masked until they satisfy the criterion as decoding progresses.

%% file: sections/method.tex
\section{Method}
\label{sec:method}
% The standard decoding process commits positions only after they satisfy a per-position criterion, assigning the greedy prediction. However, 

% Fig.~\ref{fig:ripple} (right) reveals that proactively committing a pivot position in the mid-entropy regime induces the strongest downstream uncertainty reduction. This regime is also where the correct token is often not the top-1 prediction, motivating us to design ripple-pivot search beyond greedy decoding. 

% We propose the ripple-pivot search, a novel training-free decoding method detailed below.
% making greedy assignment unreliable for early commitment. Exploiting this opportunity therefore requires addressing two questions: \emph{where} to intervene and \emph{what} token to assign.

\subsection{\methodname{}}
\label{sec:method-algorithm}

% \paragraph{Overview.} \methodshort{} introduces a per-step pivot search on top of the standard decoding process (instantiated with confidence-aware decoding in our experiments): it proactively identifies a mid-entropy position for early commitment and determines its token assignment by evaluating plausible candidates against downstream decoding benefit (Figure~\ref{fig:overview}, left). This mechanism comprises two complementary components: \emph{pivot selection} (where to intervene) and \emph{lookahead scoring} (what to commit), described in turn below.
Our RPS introduces a per-step pivot search on top of the standard decoding process (instantiated with confidence-aware decoding in our experiments) as illustrated in Fig.~\ref{fig:overview}, consisting of \emph{pivot selection} (where to commit) and \emph{lookahead scoring} (what to commit). 
% In the following, we provide the details of these two parts.

\paragraph{Pivot selection.}
Pivot selection must balance benefit and cost: pivots should lie in the mid-entropy regime, where early resolution is most beneficial, while keeping lookahead token search compact. Although such positions are uncertain, the correct token typically remains among the top-ranked candidates. Motivated by this observation, \methodshort{} truncates each position’s support to the top-$k_{\max}$ tokens, discarding most vocabulary items irrelevant to the decision, which is formulated as follows
\begin{equation}  \label{eq:selection}
  i^\star
  =
  \underset{i \in \mathcal{M}:\, \mu_i \geq \tau_{\text{pivot}}}{\arg\max}
  \left\{
    -\sum_{v \in \mathcal{T}_i} p_i(v)\log p_i(v)
  \right\},
  \qquad
  \mu_i = \sum_{v \in \mathcal{T}_i} p_i(v),
\end{equation}
where the truncated support $\mathcal{T}_i \subseteq \mathcal{V}$ contains the $k_{\max}$ tokens with the highest probabilities under predictive distribution $p_i$, and $\mu_i$ denotes the corresponding retained probability mass. With Eq.~(\ref{eq:selection}), the probability-mass constraint $(\mu_i \geq \tau_{\text{pivot}})$ excludes positions in the high-entropy regime, and maximizing truncated entropy over the retained positions avoids those that are already nearly determined. The resulting pivot therefore naturally falls in the mid-entropy regime to propagate useful information. If no position satisfies the probability-mass constraint, all remaining masked positions are still in a highly uncertain state where reliable intervention is infeasible. In this case, \methodshort{} skips pivot search at the current step and proceeds with standard decoding alone.

\begin{figure*}[t]
\centering
\includegraphics[width=\textwidth]{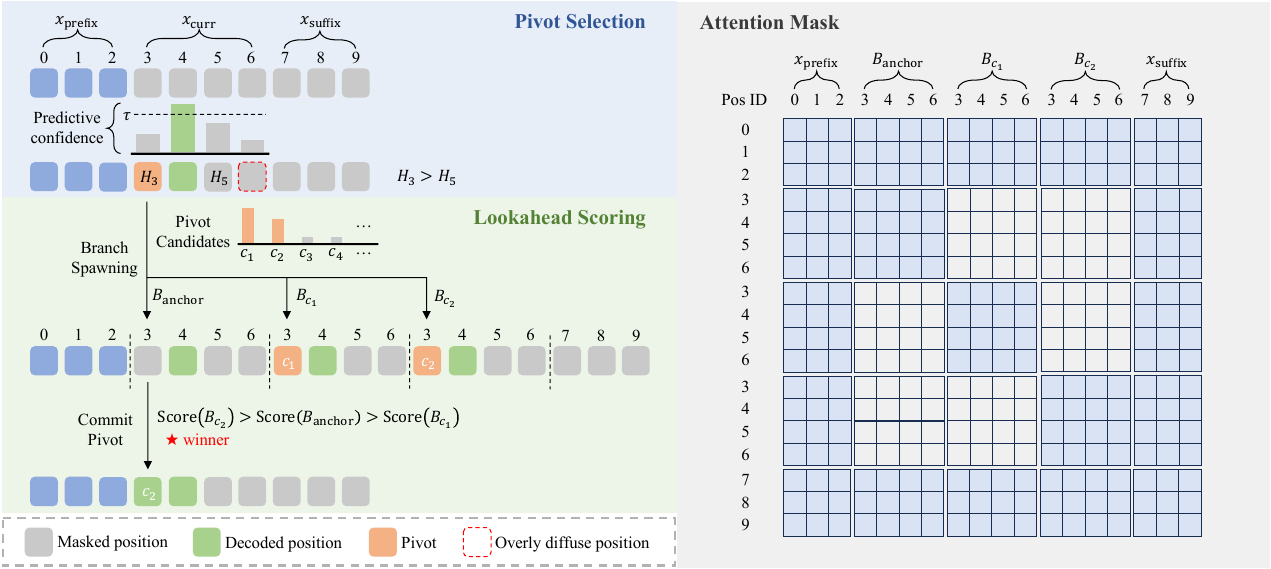}
\vspace{-15pt}
\caption{\textbf{Overview of {\methodshort}.} \textbf{Left}: One-step example of {\methodshort} decoding. The upper region illustrates pivot selection, while the lower region shows lookahead scoring. \textbf{Right}: Customized attention mask for lookahead forward pass. The packed sequence contains the shared context with candidate branches. Light-blue cells indicate allowed attention and blank cells indicate blocked attention.}
\vspace{-10pt}
\label{fig:overview}
\end{figure*}

\paragraph{Lookahead scoring.}
With the pivot $i^\star$ identified, \methodshort{} then constructs an adaptive candidate set to determine the token assignment, where the set is defined as $\mathcal{C}=\left\{v|p_{i^\star}(v) \geq r*P_{i^\star}^{\text{max}},~\forall v\in\mathcal{T}_{i^\star}\right\}\cup \{[\texttt{MASK}]\}$ by retaining  tokens that satisfy the reachability ratio $r$ relative to the top-1 probability. The reason that we include $[\texttt{MASK}]$ here is to allow the pivot to remain masked. As shown in Fig.~\ref{fig:overview}, we build a corresponding lookahead branch $B_c$ for each candidate token $c \in \mathcal{C}$ by assigning $c$ to $i^\star$ (the branch induced by assigning $[\texttt{MASK}]$ is referred to as the anchor branch $B_{\text{anchor}}$), and all branches are evaluated jointly in a single lookahead forward pass with an isolated attention mask. Finally, \methodshort{} selects the token assignment for the pivot based on the following equation:
\begin{equation}
  c^\star
  =
  \underset{c \in \mathcal{C}}{\arg\max}
  \left\{
    -\frac{1}{|\mathcal{M}| - 1}
    \sum_{i \in \mathcal{M} \setminus \{i^\star\}}
    H\!\left(p_i^c\right)
    + \lambda \log p_{\text{anchor}}(c)
  \right\},
  \label{eq:score}
\end{equation}
% where $p_i^c$ denotes the predictive distribution at position $i$ under the branch induced by assigning $c$ at the pivot, $p_{\text{anchor}}(c)$ denotes the anchor branch's probability for $c$ at the pivot, and $\lambda \geq 0$ is the plausibility weight balancing current plausibility against future benefit.
where $p_i^c$ denotes the predictive distribution at position $i$ under the branch assigning $c$ at the pivot, $p_{\text{anchor}}(c)$ denotes the anchor branch's probability for $c$ at the pivot, and is set to $P_{\text{anchor}}^{\max}$ when $c=[\texttt{MASK}]$, $\lambda \geq 0$ is the plausibility weight balancing current plausibility against future benefit.

In Eq.~(\ref{eq:score}), the first addition term inside the braces is the mean entropy over the remaining masked positions of each branch, reflecting how much entropy the pivot assignment reduces. Specifically, lower is better, signaling a stronger \effectname{} for more parallel commits in subsequent steps. The second addition term inside the braces acts as plausibility regularization to avoid trivial tokens---for instance, a premature end-of-sequence prediction may suppress downstream entropy simply by collapsing future uncertainty. Moreover, as the anchor branch keeps the pivot masked, $p_{\text{anchor}}(c)$ is evaluated without conditioning on the assignment, providing an useful independent quality signal.
% For the $[\texttt{MASK}]$ candidate, the pivot remains unresolved. We assign its plausibility term the log-probability of the anchor branch's top-1 prediction at the pivot, intentionally giving abstention a systematic advantage. This reflects the conservative default of postponing commitment: a candidate is assigned only when its expected downstream benefit is sufficient to justify early commitment.

\paragraph{Integration with standard decoding.}
At each decoding step, the standard commit rule first produces a commit set $\mathcal{S}$ and assigns greedy predictions to these positions. \methodshort{} then performs its pivot search on the remaining masked positions $\mathcal{M}$: it selects the pivot, constructs the anchor and candidate branches, and evaluates them jointly in a single forward pass using the branch-isolating attention mask in Fig~\ref{fig:overview}. Each branch is isolated from all others so that it faithfully simulates what the model would predict if only that particular token were assigned at the pivot. To avoid redundant computation, the predictive distributions from the selected branch are carried over to the next decoding step. The full per-step procedure is summarized in Appendix~\ref{app:algorithm}.

\subsection{Analysis of the Lookahead Objective}
\label{sec:objective-analysis}

We next analyze the lookahead objective in Eq.~(\ref{eq:score}). This analysis does not attempt to derive the empirically motivated pivot rule; instead, it characterizes how the two scoring terms relate to the speed--quality trade-off after a pivot has been selected. For brevity, write the downstream-position count $n=|\mathcal M|-1$ and the mean downstream entropy
$\bar H_c=\frac{1}{n}\sum_{i\in\mathcal M\setminus\{i^\star\}}H(p_i^c)$ for branch $c$.

\begin{proposition}[Entropy-certified parallelism]
\label{prop:entropy-parallelism}
Assume $n>0$. For a confidence threshold $\tau\in[1/2,1)$, define the eligible-commit count $N_\tau(c)=\sum_{i\in\mathcal M\setminus\{i^\star\}} \mathbb{I}\!\left[\max_v p_i^c(v)\geq\tau\right]$.
% \begin{equation*}
%   N_\tau(c)=\sum_{i\in\mathcal M\setminus\{i^\star\}}
%   \mathbb{I}\!\left[\max_v p_i^c(v)\geq\tau\right].
% \end{equation*}
With binary entropy $h(\tau)=-\tau\log\tau-(1-\tau)\log(1-\tau)$,
\begin{equation}
  N_\tau(c)
  \geq
  \max\!\left\{0,\;
  n-\left\lfloor\frac{n\bar H_c}{h(\tau)}\right\rfloor
  \right\}.
  \label{eq:parallelism-bound}
\end{equation}
\end{proposition}

$N_\tau(c)$ counts the positions eligible for commitment by a confidence-aware decoder in the next step. Proposition~\ref{prop:entropy-parallelism} establishes that reducing the mean downstream entropy $\bar H_c$ monotonically tightens a certified lower bound on this number. The future-benefit term in Eq.~(\ref{eq:score}) therefore has a direct speed interpretation rather than serving only as a generic uncertainty heuristic. The guarantee is deliberately conservative: it lower-bounds one-step commit opportunities but neither predicts the exact number of commits nor directly bounds end-to-end NFE.

\begin{proposition}[Plausibility-adjusted selection margin]
\label{prop:selection-margin}
Let $a_c$ denote the effective anchor plausibility used by the scoring rule, where $a_c=p_{\text{anchor}}(c)$ if $c\neq[\texttt{MASK}]$, otherwise $a_c = \max_v p_{\text{anchor}}(v)$ when $c=[\texttt{MASK}]$.
% \begin{equation*}
%   a_c=
%   \begin{cases}
%     p_{\text{anchor}}(c), & c\neq[\texttt{MASK}],\\
%     \max_v p_{\text{anchor}}(v), & c=[\texttt{MASK}],
%   \end{cases}
% \end{equation*}
Maximizing Eq.~(\ref{eq:score}) is equivalent to
\begin{equation}
  c^\star=\arg\min_{c\in\mathcal C}\left\{\bar H_c-\lambda\log a_c\right\}.
  \label{eq:lagrangian-score}
\end{equation}
Moreover, for any $c,d\in\mathcal C$, candidate $c$ scores at least as high as $d$ if and only if
\begin{equation}
  \bar H_d-\bar H_c
  \geq \lambda\log\frac{a_d}{a_c}.
  \label{eq:selection-margin}
\end{equation}
\end{proposition}

Eq.~(\ref{eq:lagrangian-score}) is a Lagrangian relaxation of minimizing downstream entropy under a candidate-surprisal budget. Proposition~\ref{prop:selection-margin} makes the resulting safeguard explicit: relative to a more plausible candidate, a less plausible candidate must compensate for its plausibility deficit with a proportionally larger entropy reduction. When $d=[\texttt{MASK}]$, the same margin governs whether RPS commits or abstains. Together, Propositions~\ref{prop:entropy-parallelism} and~\ref{prop:selection-margin} separate the two roles of the scoring function: the entropy term promotes certifiable next-step parallelism, while the anchor term controls the evidence required to take that acceleration opportunity. Appendix~\ref{app:objective-proof} provides both proofs.

\subsection{Discussion}
\label{sec:method-discussion}

As compared in Fig.~\ref{fig:ripple} (left), although \methodshort{} shares the high-level goal of accelerating dLLM decoding with lookahead evaluation, recent lookahead methods~\citep{DBLP:journals/corr/abs-2512-16229,DBLP:journals/corr/abs-2511-21103} mainly refine \emph{where} to commit: LoPA searches high-confidence residual positions, whereas ETE targets positions near a prescribed confidence level. \methodshort{} instead selects mid-entropy pivots via truncated entropy. More importantly, unlike LoPA and ETE, which retain greedy top-1 assignment, \methodshort{} also refines \emph{what} to commit by lookahead evaluation over plausible token candidates, motivated by our finding that the correct token is often non-top-1 where the \effectname{} is strongest.

% \methodshort{} shares the high-level goal of accelerating dLLM decoding with several contemporaneous lookahead-based methods~\citep{DBLP:journals/corr/abs-2512-16229,DBLP:journals/corr/abs-2511-21103}. Both LoPA and ETE refine the \emph{where} decision by using lookahead forward passes to identify positions worth committing early, but they define their position candidates differently. LoPA explores the highest-confidence positions that remain after the standard decoding update, whereas ETE targets positions near a prescribed medium-confidence level. \methodshort{} instead identifies the mid-entropy regime from the predictive distribution: it excludes positions whose probability mass is too diffuse for reliable candidate search and then selects the position with the highest truncated entropy. More importantly, neither LoPA nor ETE revisits the \emph{what} decision: once a position is selected, it is directly assigned its top-1 prediction. As our oracle analysis shows, the correct token is often not top-1 precisely in the mid-entropy regime where the \effectname{} is strongest. \methodshort{} therefore complements entropy-based pivot selection with lookahead evaluation over plausible token assignments and a built-in plausibility safeguard.

%% file: sections/experiments.tex
\section{Experiments}
\label{sec:experiments}

\paragraph{Setup.}
\label{sec:exp-setup}
We evaluate two representative dLLM families: LLaDA family, which is trained from scratch, and Dream family, which is adapted from an autoregressive language model. Our evaluation covers mathematical reasoning (GSM8K~\citep{DBLP:journals/corr/abs-2110-14168} and MATH500~\citep{DBLP:conf/iclr/LightmanKBEBLLS24}) and code generation (HumanEval~\citep{DBLP:journals/corr/abs-2107-03374} and MBPP~\citep{DBLP:journals/corr/abs-2108-07732}). We compare {\methodshort} with the standard one-token-per-step \textbf{Default} decoder and five parallel decoding or sampling baselines: \textbf{Confidence}~\citep{DBLP:journals/corr/abs-2505-22618}, \textbf{KLASS}~\citep{DBLP:conf/nips/KimHJPY25}, \textbf{EB-Sampler}~\citep{DBLP:conf/nips/Ben-HamuGSNK25}, \textbf{WINO}~\citep{DBLP:journals/corr/abs-2507-18578}, and \textbf{LoPA}~\citep{DBLP:journals/corr/abs-2512-16229}. For evaluation,
we use lm-evaluation-harness\footnote{\href{https://github.com/EleutherAI/lm-evaluation-harness}{https://github.com/EleutherAI/lm-evaluation-harness}} with its standard task implementations and prompting configurations: 5-shot for GSM8K, 4-shot for MATH500, 0-shot for HumanEval, and 3-shot for MBPP. Unless stated otherwise, the generation length is 256 and the block length is 32. We measure generation quality by accuracy and efficiency by both the number of function evaluations (NFE), which captures sequential model evaluations, and tokens per second (TPS), which captures end-to-end throughput. Speedups are computed relative to Default under the same model, task, and generation length.

\paragraph{Implementation details.}
We build on the Fast-dLLM inference stack~\citep{DBLP:journals/corr/abs-2505-22618} and evaluate LLaDA-8B-Instruct,
% \footnote{\href{https://huggingface.co/GSAI-ML/LLaDA-8B-Instruct}{https://huggingface.co/GSAI-ML/LLaDA-8B-Instruct}}, 
Dream-v0-Instruct-7B,
% \footnote{\href{https://huggingface.co/Dream-org/Dream-v0-Instruct-7B}{https://huggingface.co/Dream-org/Dream-v0-Instruct-7B}}, 
and LLaDA-1.5,
% \footnote{\href{https://huggingface.co/GSAI-ML/LLaDA-1.5}{https://huggingface.co/GSAI-ML/LLaDA-1.5}}, 
for 12 model--benchmark configurations in total. For brevity, we refer to LLaDA-8B-Instruct and Dream-v0-Instruct-7B as LLaDA and Dream, respectively, throughout this section. We report the main-text results on LLaDA and Dream, with additional results on LLaDA-1.5 provided in Appendix~\ref{app:llada15}. Unless otherwise noted, {\methodshort} uses $\tau=0.9$,  $k_{\max}=10$, reachability ratio $r=0.1$, and probability-mass threshold $\tau_{\text{pivot}}=0.9$ for LLaDA and 0.95 for Dream. We select the plausibility weight $\lambda$ from the interval $[0.1,0.5]$ identified in \S\ref{sec:exp-aw-sweep}. Appendix~\ref{app:implementation} reports baseline configurations and hardware details.

\subsection{Main Results}
\label{sec:exp-main}

\paragraph{Comparison with parallel decoding baselines.}
We compare {\methodshort} with the default one-token-per-step decoder and several state-of-the-art parallel decoding methods across two model families and four benchmarks, with results reported in Table~\ref{tab:main-results}. {\methodshort} remains in the highest-throughput regime among parallel decoders, delivering 4.24--9.80$\times$ TPS speedup over Default while largely preserving generation quality. In several settings, acceleration even comes with improved accuracy: on LLaDA, {\methodshort} exceeds Default by 2.2\% on MBPP, and on HumanEval it improves accuracy by 1.22 \% for both LLaDA and Dream. These results demonstrate that jointly selecting a mid-entropy pivot and evaluating plausible token assignments with an explicit plausibility safeguard enables aggressive parallel commitment while preserving or improving generation quality in most settings.
Among the parallel baselines, LoPA is closest to {\methodshort} in decoding speed, as it likewise uses lookahead to guide early commitment. However, LoPA consistently loses accuracy relative to Default on the code-generation benchmarks. The gap is especially clear on HumanEval, where {\methodshort} outperforms LoPA by 4.27 \% on LLaDA and 5.49 \% on Dream at comparable throughput. Section~\ref{sec:exp-case} provides a dedicated failure-mode analysis of this behavior.

\begin{table}[t]
  \centering
  \caption{Results on LLaDA and Dream. \textbf{Bold} indicates the best result.}
  \label{tab:main-results}
  \label{tab:main-llada8b}
  \label{tab:main-dream}
  \vspace{-10pt}
  \scriptsize
  \setlength{\tabcolsep}{2.4pt}
  \resizebox{\textwidth}{!}{%
  \begin{tabular}{@{}ll ccccc ccccc@{}}
  \toprule
  \multirow{2}{*}{\textbf{Benchmark}} & \multirow{2}{*}{\textbf{Method}}
  & \multicolumn{5}{c}{\textbf{LLaDA-8B-Instruct}}
  & \multicolumn{5}{c}{\textbf{Dream-v0-Instruct-7B}} \\
  \cmidrule(lr){3-7}\cmidrule(lr){8-12}
  & & \textbf{Acc} & \textbf{NFE} & \textbf{NFE Sp.} & \textbf{TPS} & \textbf{TPS Sp.}
  & \textbf{Acc} & \textbf{NFE} & \textbf{NFE Sp.} & \textbf{TPS} & \textbf{TPS Sp.} \\
  \midrule
  \multirow{7}{*}{\shortstack[l]{GSM8K\\(5-shot)}}
  & Default & 79.00 & 256.0 & 1.00$\times$ & 4.02 & 1.00$\times$ & 78.70 & 256.0 & 1.00$\times$ & 5.14 & 1.00$\times$ \\
  & Confidence & \textbf{79.45} & 77.80 & 3.29$\times$ & 13.21 & 3.29$\times$ & 77.33 & 66.05 & 3.88$\times$ & 19.16 & 3.73$\times$ \\
  & KLASS & 79.30 & 131.93 & 1.94$\times$ & 7.88 & 1.96$\times$ & \textbf{79.38} & 120.87 & 2.12$\times$ & 10.47 & 2.04$\times$ \\
  & EB-Sampler & \textbf{79.45} & 87.73 & 2.92$\times$ & 11.88 & 2.96$\times$ & 79.30 & 108.15 & 2.37$\times$ & 11.86 & 2.31$\times$ \\
  & WINO & 78.24 & 57.37 & 4.46$\times$ & 17.70 & 4.40$\times$ & 76.12 & 61.90 & 4.14$\times$ & 19.43 & 3.78$\times$ \\
  & LoPA & 78.70 & 41.95 & 6.10$\times$ & 22.98 & 5.68$\times$ & 75.66 & \textbf{34.62} & \textbf{7.40$\times$} & \textbf{31.31} & \textbf{6.09$\times$} \\
  & {\methodshort} & 79.15 & \textbf{35.73} & \textbf{7.17$\times$} & \textbf{26.66} & \textbf{6.63$\times$} & 78.62 & 41.74 & 6.13$\times$ & 28.21 & 5.49$\times$ \\
  \midrule
  \multirow{7}{*}{\shortstack[l]{HumanEval\\(0-shot)}}
  & Default & 41.46 & 256.0 & 1.00$\times$ & 14.24 & 1.00$\times$ & 57.32 & 256.0 & 1.00$\times$ & 9.13 & 1.00$\times$ \\
  & Confidence & 42.07 & 78.83 & 3.25$\times$ & 45.93 & 3.23$\times$ & 58.54 & 82.98 & 3.08$\times$ & 28.51 & 3.12$\times$ \\
  & KLASS & 41.46 & 128.82 & 1.99$\times$ & 28.75 & 2.02$\times$ & \textbf{60.98} & 125.45 & 2.04$\times$ & 19.26 & 2.11$\times$ \\
  & EB-Sampler & 41.46 & 91.46 & 2.80$\times$ & 40.16 & 2.82$\times$ & 60.37 & 110.01 & 2.33$\times$ & 21.98 & 2.41$\times$ \\
  & WINO & 40.24 & 73.67 & 3.47$\times$ & 47.83 & 3.36$\times$ & 57.93 & 79.37 & 3.23$\times$ & 26.68 & 2.92$\times$ \\
  & LoPA & 38.41 & 40.25 & 6.36$\times$ & 76.42 & 5.37$\times$ & 53.05 & \textbf{40.36} & \textbf{6.34$\times$} & \textbf{43.02} & \textbf{4.71$\times$} \\
  & {\methodshort} & \textbf{42.68} & \textbf{37.22} & \textbf{6.88$\times$} & \textbf{82.98} & \textbf{5.83$\times$} & 58.54 & 48.33 & 5.30$\times$ & 41.32 & 4.53$\times$ \\
  \midrule
  \multirow{7}{*}{\shortstack[l]{MATH500\\(4-shot)}}
  & Default & 38.40 & 256.0 & 1.00$\times$ & 6.02 & 1.00$\times$ & 44.60 & 256.0 & 1.00$\times$ & 7.84 & 1.00$\times$ \\
  & Confidence & 38.60 & 98.85 & 2.59$\times$ & 15.47 & 2.57$\times$ & \textbf{45.60} & 93.90 & 2.73$\times$ & 21.17 & 2.70$\times$ \\
  & KLASS & 38.20 & 157.05 & 1.63$\times$ & 9.92 & 1.65$\times$ & 44.20 & 142.56 & 1.80$\times$ & 14.28 & 1.82$\times$ \\
  & EB-Sampler & 38.60 & 210.50 & 1.22$\times$ & 7.45 & 1.24$\times$ & 44.20 & 100.29 & 2.55$\times$ & 20.36 & 2.60$\times$ \\
  & WINO & \textbf{38.80} & 73.15 & 3.50$\times$ & 20.75 & 3.45$\times$ & 44.80 & 82.60 & 3.10$\times$ & 22.31 & 2.84$\times$ \\
  & LoPA & 38.20 & 70.36 & 3.64$\times$ & 21.21 & 3.53$\times$ & 44.20 & 64.39 & 3.98$\times$ & 27.62 & 3.52$\times$ \\
  & {\methodshort} & 38.20 & \textbf{48.32} & \textbf{5.30$\times$} & \textbf{29.09} & \textbf{4.83$\times$} & 44.80 & \textbf{55.61} & \textbf{4.60$\times$} & \textbf{33.27} & \textbf{4.24$\times$} \\
  \midrule
  \multirow{7}{*}{\shortstack[l]{MBPP\\(3-shot)}}
  & Default & 30.60 & 256.0 & 1.00$\times$ & 3.99 & 1.00$\times$ & \textbf{58.80} & 256.0 & 1.00$\times$ & 2.40 & 1.00$\times$ \\
  & Confidence & 30.60 & 69.82 & 3.67$\times$ & 14.72 & 3.69$\times$ & 56.20 & 39.61 & 6.46$\times$ & 14.71 & 6.13$\times$ \\
  & KLASS & 30.60 & 111.32 & 2.30$\times$ & 9.41 & 2.36$\times$ & 55.20 & 61.58 & 4.16$\times$ & 8.87 & 3.70$\times$ \\
  & EB-Sampler & 30.80 & 79.64 & 3.22$\times$ & 13.07 & 3.28$\times$ & 55.00 & 38.34 & 6.68$\times$ & 14.17 & 5.90$\times$ \\
  & WINO & 30.00 & 67.63 & 3.79$\times$ & 15.06 & 3.77$\times$ & 56.40 & 42.03 & 6.09$\times$ & 13.11 & 5.46$\times$ \\
  & LoPA & 29.20 & 37.83 & 6.77$\times$ & 23.25 & 5.83$\times$ & 55.60 & 25.19 & 10.16$\times$ & 21.75 & 9.06$\times$ \\
  & {\methodshort} & \textbf{32.80} & \textbf{29.92} & \textbf{8.56$\times$} & \textbf{25.10} & \textbf{6.29$\times$} & 56.80 & \textbf{23.81} & \textbf{10.76$\times$} & \textbf{23.52} & \textbf{9.80$\times$} \\
  \bottomrule
  \end{tabular}%
  }
  \vspace{-10pt}
\end{table}

\begin{table}[t]
\centering
\caption{Generation-length robustness on LLaDA and Dream. 
% \textbf{Bold} indicates the best result.
}
\label{tab:length}
\vspace{-10pt}
\scriptsize
\setlength{\tabcolsep}{2.4pt}
\resizebox{\textwidth}{!}{%
\begin{tabular}{@{}lll ccc >{\columncolor{black!5}}c >{\columncolor{black!5}}c >{\columncolor{black!5}}c ccc@{}}
\toprule
\multirow{2}{*}{\textbf{Model}} & \multirow{2}{*}{\textbf{Benchmark}} & \multirow{2}{*}{\textbf{Method}}
& \multicolumn{3}{c}{$\boldsymbol{L=128}$} & \multicolumn{3}{>{\columncolor{black!5}}c}{$\boldsymbol{L=256}$} & \multicolumn{3}{c}{$\boldsymbol{L=512}$} \\
\cmidrule(lr){4-6}\cmidrule(lr){7-9}\cmidrule(lr){10-12}
& & & \textbf{Acc} & \textbf{NFE Sp.} & \textbf{TPS Sp.}
& \textbf{Acc} & \textbf{NFE Sp.} & \textbf{TPS Sp.}
& \textbf{Acc} & \textbf{NFE Sp.} & \textbf{TPS Sp.} \\
\midrule
\multirow{8}{*}{\shortstack[l]{LLaDA-8B-\\Instruct}}
& \multirow{4}{*}{\shortstack[l]{HumanEval\\(0-shot)}}
& Default & \textbf{30.49} & 1.00$\times$ & 1.00$\times$ & 41.46 & 1.00$\times$ & 1.00$\times$ & \textbf{43.29} & 1.00$\times$ & 1.00$\times$ \\
& & Confidence & 29.27 & 3.59$\times$ & 3.49$\times$ & 42.07 & 3.25$\times$ & 3.23$\times$ & 41.46 & 3.15$\times$ & 3.15$\times$ \\
& & LoPA & 22.56 & 4.79$\times$ & 4.39$\times$ & 38.41 & 6.36$\times$ & 5.37$\times$ & 40.85 & 6.32$\times$ & 5.19$\times$ \\
& & {\methodshort} & 28.66 & \textbf{6.49$\times$} & \textbf{5.26$\times$} & \textbf{42.68} & \textbf{6.88$\times$} & \textbf{5.83$\times$} & 42.07 & \textbf{6.48$\times$} & \textbf{5.45$\times$} \\
\cmidrule(lr){2-12}
& \multirow{4}{*}{\shortstack[l]{MATH500\\(4-shot)}}
& Default & 34.40 & 1.00$\times$ & 1.00$\times$ & 38.40 & 1.00$\times$ & 1.00$\times$ & 42.20 & 1.00$\times$ & 1.00$\times$ \\
& & Confidence & 34.20 & 2.13$\times$ & 2.13$\times$ & \textbf{38.60} & 2.59$\times$ & 2.57$\times$ & \textbf{42.40} & 3.42$\times$ & 3.38$\times$ \\
& & LoPA & 34.40 & 3.00$\times$ & 2.91$\times$ & 38.20 & 3.64$\times$ & 3.53$\times$ & 40.40 & 6.84$\times$ & 6.21$\times$ \\
& & {\methodshort} & \textbf{35.60} & \textbf{4.22$\times$} & \textbf{3.86$\times$} & 38.20 & \textbf{5.30$\times$} & \textbf{4.83$\times$} & \textbf{42.40} & \textbf{6.97$\times$} & \textbf{6.36$\times$} \\
\midrule
\multirow{8}{*}{\shortstack[l]{Dream-v0-7B-\\Instruct}}
& \multirow{4}{*}{\shortstack[l]{HumanEval\\(0-shot)}}
& Default & 58.54 & 1.00$\times$ & 1.00$\times$ & 57.32 & 1.00$\times$ & 1.00$\times$ & 53.66 & 1.00$\times$ & 1.00$\times$ \\
& & Confidence & \textbf{59.76} & 2.71$\times$ & 2.56$\times$ & \textbf{58.54} & 3.08$\times$ & 3.12$\times$ & 57.32 & 5.72$\times$ & 5.63$\times$ \\
& & LoPA & 53.66 & 4.14$\times$ & 3.30$\times$ & 53.05 & \textbf{6.34$\times$} & \textbf{4.71$\times$} & 53.66 & 9.43$\times$ & 6.90$\times$ \\
& & {\methodshort} & 58.54 & \textbf{4.20$\times$} & \textbf{3.73$\times$} & \textbf{58.54} & 5.30$\times$ & 4.53$\times$ & \textbf{58.54} & \textbf{9.77$\times$} & \textbf{8.19$\times$} \\
\cmidrule(lr){2-12}
& \multirow{4}{*}{\shortstack[l]{MATH500\\(4-shot)}}
& Default & 42.00 & 1.00$\times$ & 1.00$\times$ & 44.60 & 1.00$\times$ & 1.00$\times$ & 45.60 & 1.00$\times$ & 1.00$\times$ \\
& & Confidence & \textbf{42.80} & 2.08$\times$ & 2.09$\times$ & \textbf{45.60} & 2.73$\times$ & 2.70$\times$ & 47.80 & 4.31$\times$ & 4.30$\times$ \\
& & LoPA & 40.60 & 3.02$\times$ & 2.65$\times$ & 44.20 & 3.98$\times$ & 3.52$\times$ & 46.20 & 6.54$\times$ & 5.79$\times$ \\
& & {\methodshort} & 41.40 & \textbf{3.39$\times$} & \textbf{2.99$\times$} & 44.80 & \textbf{4.60$\times$} & \textbf{4.24$\times$} & \textbf{49.20} & \textbf{7.15$\times$} & \textbf{6.62$\times$} \\
\bottomrule
\end{tabular}%
}
\vspace{-10pt}
\end{table}

\paragraph{Robustness to generation length.}
The preceding experiments establish a strong quality--efficiency trade-off for {\methodshort} at the default generation length of 256. We next examine whether this advantage persists across different generation lengths. As shown in Table~\ref{tab:length}, {\methodshort} maintains the strongest quality--efficiency trade-off among the compared baselines across all evaluated lengths. A model-specific exception is HumanEval at $L=128$, where all LLaDA decoders exhibit lower accuracy than at longer lengths, while Dream does not show the same degradation. Inspection of the generated programs indicates that the LLaDA outputs are generally syntactically complete but often implement overly simplified or incomplete logic, suggesting that the restricted budget affects solution formation rather than merely truncating the code. Under this constraint, LoPA's aggressive early commitments further amplify the quality loss, reducing accuracy from 30.49 for Default to 22.56. In contrast, {\methodshort} retains 28.66 accuracy while achieving 6.49$\times$ NFE and 5.26$\times$ TPS speedups, limiting the additional degradation caused by parallel decoding.

\subsection{Hyperparameter Ablations}
\label{sec:exp-hparams}

\paragraph{Plausibility-weight sensitivity.}
\label{sec:exp-aw-sweep}
We isolate the contribution of the plausibility safeguard by sweeping $\lambda \in \{0.0,0.1,0.2,0.3,0.4,0.5\}$ on LLaDA GSM8K, as shown in Fig.~\ref{fig:ablation-summary} (left). Enabling the safeguard with any positive $\lambda$ substantially improves accuracy over entropy-only scoring ($\lambda=0$); in particular, $\lambda=0.1$ yields a 2.1\% gain. This consistent improvement demonstrates that branch-external plausibility is necessary to prevent lookahead from favoring decisive but incorrect token assignments. Meanwhile, accuracy varies by only 1.1\% across $\lambda\in[0.1,0.5]$, with little change in TPS speedup. The scoring function is therefore robust once the safeguard is enabled and does not require extensive scenario-specific hyperparameter tuning.

\begin{figure*}[t]
  \vspace{-5pt}
  \centering

  % Left panel
  \begin{minipage}[t]{0.35\textwidth}
    \vspace{0pt}
    \centering
    \includegraphics[width=\linewidth]{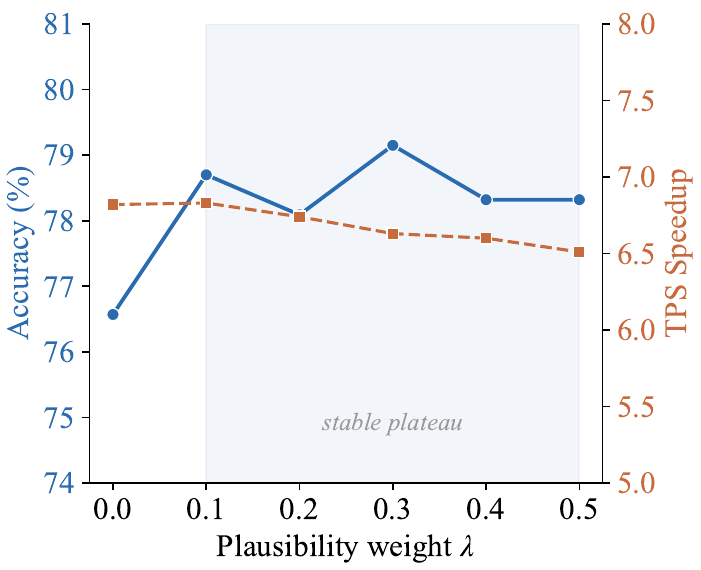}
  \end{minipage}
  \hfill
  % Middle panel
  \begin{minipage}[t]{0.37\textwidth}
    \vspace{0pt}
    \centering
    \includegraphics[width=\linewidth]{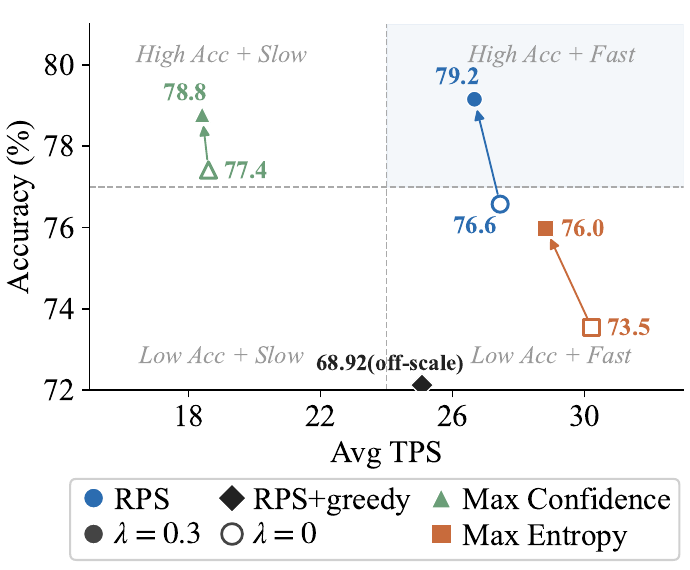}
  \end{minipage}
  \hfill
  % Right panel
  \begin{minipage}[t]{0.23\textwidth}
    \vspace{0pt}
    \centering
    \includegraphics[width=\linewidth]{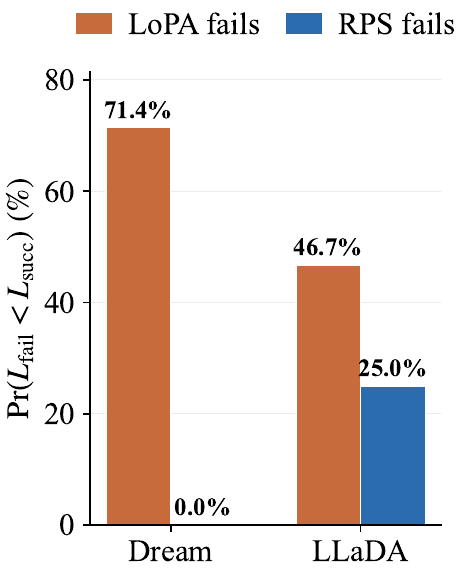}
  \end{minipage}
  \vspace{-5pt}
  \caption{
      \textbf{Left:} Accuracy and TPS speedup of RPS on GSM8K with LLaDA under different plausibility weights $\lambda$.
      \textbf{Middle:} Accuracy and TPS trade-off of different pivot-selection and scoring strategies on GSM8K with LLaDA
      \textbf{Right:} Failure-length comparison on HumanEval for Dream and
      LLaDA. For cases where exactly one of LoPA and {\methodshort} succeeds, we report the fraction for which the failing completion contains fewer non-empty source lines than the successful completion.
  }
  \label{fig:ablation-summary}

  \vspace{-15pt}
\end{figure*}

\paragraph{Search-space hyperparameters.}
\begin{wraptable}{r}{0.52\textwidth}
  \centering
  \vspace{-12pt}
  \caption{Hyperparameter sensitivity on LLaDA GSM8K. One parameter is varied at a time; \textbf{bold} marks the setting used in the main experiments.}
  \label{tab:hparam-sweep}
  \vspace{-8pt}
  \scriptsize
  \setlength{\tabcolsep}{1.8pt}
  \begin{tabular}{@{}llrrrr@{}}
  \toprule
  \textbf{Hyperparameter} & \textbf{Value} & \textbf{Acc (\%)} & \textbf{Avg. NFE} & \textbf{TPS} & \textbf{Avg. cand.} \\
  \midrule
  \multirow{3}{*}{$k_{\max}$}
  & 5 & 78.70 & 42.09 & 23.24 & 3.00 \\
  & \textbf{10} & \textbf{79.15} & \textbf{35.73} & \textbf{26.66} & \textbf{3.87} \\
  & 20 & 74.91 & 32.02 & 28.30 & 5.16 \\
  \midrule
  \multirow{4}{*}{$r$}
  & 0.05 & 78.62 & 35.15 & 26.31 & 4.77 \\
  & \textbf{0.10} & \textbf{79.15} & \textbf{35.73} & \textbf{26.66} & \textbf{3.87} \\
  & 0.20 & 78.70 & 37.51 & 26.06 & 3.00 \\
  & 0.30 & 77.41 & 39.93 & 24.90 & 2.55 \\
  \midrule
  \multirow{5}{*}{$\tau_{\text{pivot}}$}
  & 0.60 & 78.17 & 31.12 & 29.26 & 5.05 \\
  & 0.70 & 77.33 & 31.38 & 28.93 & 5.02 \\
  & 0.80 & 78.62 & 32.26 & 28.54 & 4.70 \\
  & \textbf{0.90} & \textbf{79.15} & \textbf{35.73} & \textbf{26.66} & \textbf{3.87} \\
  & 0.95 & 77.86 & 40.62 & 23.90 & 3.20 \\
  \bottomrule
  \end{tabular}
  \vspace{-8pt}
  \end{wraptable}
We further analyze the three hyperparameters governing the pivot-search space (Table~\ref{tab:hparam-sweep}). The candidate budget $k_{\max}$ caps the support: overly large values admit weak candidates and hurt accuracy, while overly small values limit useful alternatives and slow decoding. The reachability ratio $r$ further prunes tokens relative to the top-1 probability and is robust over a broad range, though aggressive pruning may remove plausible assignments. The threshold $\tau_{\text{pivot}}$ controls pivot reliability: low values admit overly diffuse positions, whereas high values make selection too conservative and reduce speed. Since these parameters impose structural search constraints rather than task-specific preferences, we fix $k_{\max}$ and $r$ globally and use a single $\tau_{\text{pivot}}$ across all tasks within each model family, without any tuning.
% We further study the three hyperparameters that control the pivot-search space, with results reported in Table~\ref{tab:hparam-sweep}. The candidate budget $k_{\max}$ sets the upper bound of the truncated support. An overly large budget ($k_{\max}=20$) admits more weak candidate tokens and reduces accuracy by 4.24 points, whereas a smaller budget leaves accuracy nearly unchanged but noticeably slows decoding by excluding useful alternatives for early commitment. The reachability ratio $r$ complements this hard cap by pruning tokens relative to the top-1 probability; performance remains stable over a broad range, but overly aggressive pruning can discard plausible assignments. The probability-mass threshold $\tau_{\text{pivot}}$ instead determines whether a position is sufficiently concentrated for reliable intervention: a low threshold admits diffuse, high-uncertainty pivots, while a high threshold makes pivot selection overly conservative and reduces speed. These hyperparameters encode structural constraints on the search rather than task-specific preferences. We therefore keep $k_{\max}$ and $r$ fixed globally and use a single $\tau_{\text{pivot}}$ for all tasks within each model family, without task-specific tuning.

% ===========================================================================
\subsection{Further Analysis}
\label{sec:exp-analysis}

\paragraph{Pivot strategy and scoring ablation.}
\label{sec:exp-pivot-ablation}
% Fig.~\ref{fig:ablation-summary} (middle) organizes the compared configurations into four accuracy--speed regimes. Only the complete {\methodshort} design lies in the desirable high-accuracy and fast quadrant. Max Confidence retains accuracy but remains slow because it selects positions that are already nearly resolved and thus induces weaker ripple effects. Conversely, unconstrained Max Entropy is fast but inaccurate because it admits positions whose distributions are too diffuse for reliable candidate search. The proposed reachability constraint targets the tractable mid-entropy region between these extremes, while plausibility-aware scoring further lifts accuracy with negligible throughput change. Together, the pivot-selection strategy and lookahead scoring enable the best quality--speed trade-off among the compared designs.
Fig.~\ref{fig:ablation-summary} (middle) organizes the compared configurations into four accuracy--speed regimes. Only the complete {\methodshort} design lies in the desirable high-accuracy and fast quadrant. The contrast among Max Confidence, Max Entropy, and RPS indicates that the reachability-constrained RPS pivot targets a tractable mid-entropy region and delivers the strongest quality--speed trade-off. Comparing $\lambda=0$ with $\lambda=0.3$ further shows that plausibility-aware scoring improves accuracy with little affect on throughput. Finally, RPS+greedy retains the same RPS pivot selection but directly commits the top-1 token. Its lower accuracy at comparable speed indicates that greedy assignment is unreliable in the mid-entropy regime, echoing the observation in~\S\ref{sec:intro} that the correct token is often not the current top-1 prediction. Together, these comparisons highlight the complementary roles of mid-entropy pivot selection and lookahead scoring in RPS.

\paragraph{Failure-mode analysis on HumanEval.}
\label{sec:exp-case}
To better understand why {\methodshort} attains higher accuracy than LoPA on zero-shot HumanEval, we examine the programs generated by the two methods. We focus on cases where exactly one method succeeds and compare the number of non-empty source lines in the failing and successful completions. Fig.~\ref{fig:ablation-summary} (right) reveals a clear asymmetry: a shorter failing program is substantially more common when LoPA fails than when {\methodshort} fails, suggesting that LoPA is more prone to terminating a plausible-looking solution before completing the required logic. Fig.~\ref{fig:case-cache-summary} (left) illustrates this behavior on HumanEval/57 (monotonic). LoPA commits \texttt{return} early in decoding, closing the program before it accounts for monotonically decreasing inputs. {\methodshort} makes more conservative early commitments and later completes the complementary state logic, thereby avoiding this premature termination. This example reflects a broader failure mode of confidence-driven lookahead: LoPA can favor a wrong-but-decisive token because it immediately collapses downstream uncertainty. {\methodshort} instead combines cautious pivot selection with a plausibility safeguard that screens the quality of an early commitment. This distinction is especially consequential in code generation, where control-flow tokens such as \texttt{return} and \texttt{break} directly alter the execution path. An incorrect commitment can therefore produce a program that is syntactically valid and executable, yet globally incomplete or logically wrong---an error that subsequent decoding cannot easily repair.

\begin{figure*}[t]
  \centering

  % Left panel
  \begin{minipage}[t]{0.46\textwidth}
    \vspace{0pt}
    \centering
    \includegraphics[width=\linewidth]{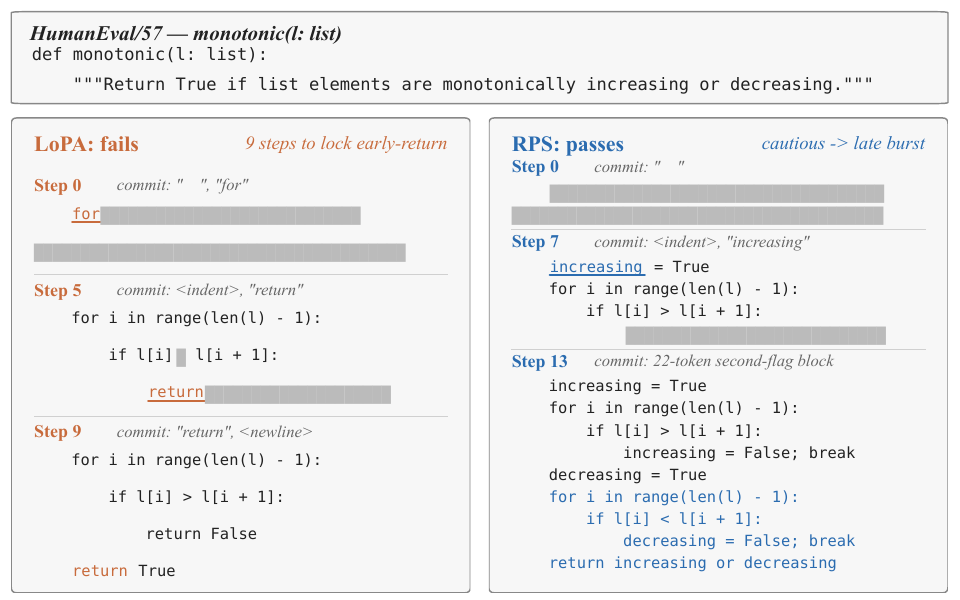}
  \end{minipage}
  \hfill
  % Right panel
  \begin{minipage}[t]{0.53\textwidth}
    \vspace{0pt}
    \centering
    \includegraphics[width=\linewidth]{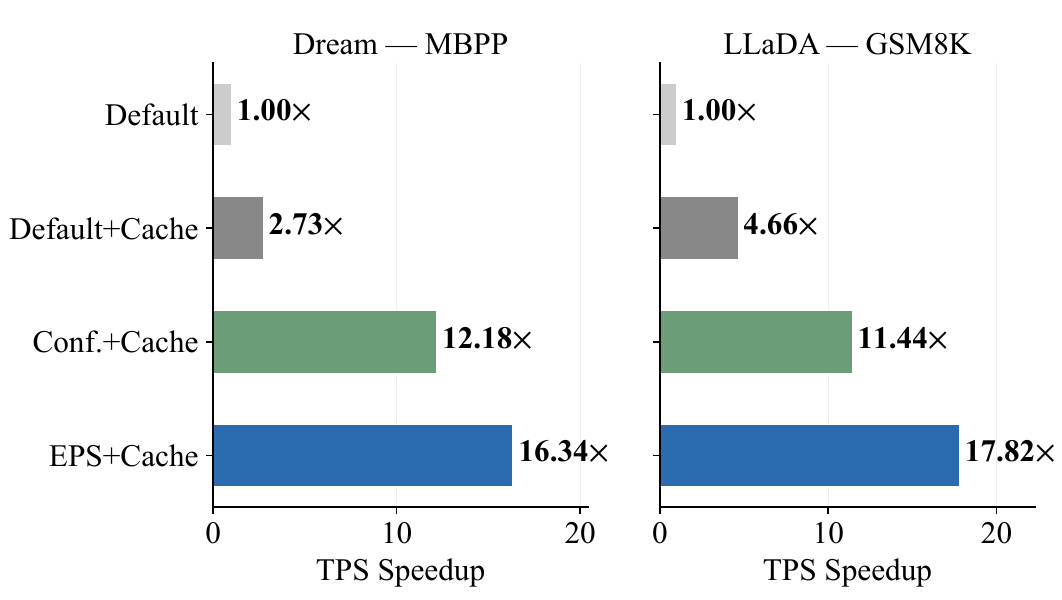}
  \end{minipage}
  \vspace{-10pt}
  \caption{
    \textbf{Left:} Decoding snapshots of LoPA and {\methodshort} on
    HumanEval/57 (\texttt{monotonic}). Coloured tokens denote newly committed
    tokens, underlined tokens denote lookahead commits, and
    ``$\blacksquare$'' denotes still-masked positions.
    \textbf{Right:} TPS speedup relative to Default and Confidence with and without
    Fast-dLLM prefix caching, evaluated on MBPP with Dream
    and GSM8K with LLaDA.
  }
  \label{fig:case-cache-summary}

  \vspace{-10pt}
\end{figure*}

\paragraph{Speedup decomposition.}
\label{sec:exp-decomposition}
\begin{wraptable}{r}{0.48\textwidth}
\centering
\small
\vspace{-12pt}
\caption{Average forward-pass breakdown per sample for {\methodshort} on LLaDA GSM8K.}
\label{tab:decomp}
\vspace{-8pt}
\setlength{\tabcolsep}{3pt}
\begin{tabular}{@{}lccc@{}}
\toprule
\textbf{Forward type} & \textbf{Avg. count} & \textbf{Avg. latency} & \textbf{Time \%} \\
\midrule
Normal & 12.3 & 219\,ms & 31.4\% \\
Lookahead & 23.4 & 252\,ms & 68.6\% \\
\midrule
\textbf{Total} & \textbf{35.7} & \textbf{241\,ms} & 100\% \\
\bottomrule
\end{tabular}
\vspace{-10pt}
\end{wraptable}
As shown in Fig.~\ref{fig:overview}, our packed lookahead evaluates all candidate branches jointly with the normal sequence in a single forward pass. Although this increases per-step computation, Table~\ref{tab:decomp} shows that a lookahead pass is only 15\% slower than a normal pass, while reducing the average number of forward passes from 77.80 under confidence decoding to 35.73 with {\methodshort}. Thus, a modest per-step overhead yields a substantial reduction in decoding iterations.
% As illustrated in Figure~\ref{fig:overview}, our lookahead forward pass concatenates all candidate branches with the normal sequence so that they can be evaluated jointly without introducing a separate forward pass for each candidate. Although this packed design naturally increases the computation within each decoding step, the breakdown in Table~\ref{tab:decomp} shows that a lookahead pass is only 15\% slower than a normal pass. In return, introducing lookahead reduces the average number of forward passes from 77.80 under Confidence decoding to 35.73 under {\methodshort}. These results validate the packed lookahead design: a modest increase in per-step cost enables a substantial reduction in the number of decoding steps.

\paragraph{Compatibility with KV caching.}
\label{sec:exp-cache}
We combine {\methodshort} with Fast-dLLM prefix caching to test whether iteration-level and per-forward optimizations are complementary. As in Fig.~\ref{fig:case-cache-summary} (right), the combination achieves the highest throughput in both settings, with up to 17.82$\times$ TPS speedup and less than 0.5\% accuracy change. This complementarity is natural: {\methodshort} reduces decoding iterations via proactive commitment, while KV caching lowers the attention cost per iteration.
% We combine {\methodshort} with the Fast-dLLM prefix cache to test whether iteration-level and per-forward optimizations compose. As shown in Figure~\ref{fig:case-cache-summary} (right), the combined method achieves the highest throughput in both evaluated settings, reaching up to 17.82$\times$ TPS speedup, while caching changes accuracy by less than 0.5 points. This compatibility follows naturally from their complementary roles: {\methodshort} reduces the number of decoding iterations through proactive commitment, whereas KV caching reduces the attention cost within each iteration.

%% file: sections/conclusion.tex
\section{Conclusion}

We presented Ripple-Pivot Search (RPS), a training-free parallel decoding
method for diffusion language models motivated by the ripple effect.
Proactively committing a pivot position in the mid-entropy regime can
substantially reduce uncertainty across the remaining masked positions,
creating more opportunities for parallel commitment in subsequent decoding
steps. RPS exploits this effect by addressing both \emph{where to commit} and \emph{what to commit}. 
It seeks mid-entropy pivots with high potential for
downstream uncertainty reduction and determines their token assignments
through lookahead evaluation over plausible candidates. Across three dLLMs
and four reasoning and code-generation benchmarks, RPS achieves
4--10$\times$ wall-clock speedup over the standard decoder while largely
preserving generation quality, and improves accuracy over the previous
lookahead baseline by up to 5.49\% while delivering higher
throughput in most settings. When combined with KV caching, RPS further
achieves up to 18$\times$ wall-clock speedup over the standard decoder.
Together, these results demonstrate the effectiveness of jointly searching
over commitment positions and token assignments for efficient dLLM decoding.

%% file: sections/acknowledgement.tex
\section*{AI use statement}

In this work, we used generative AI tools to edit the manuscript to improve readability, including grammar, wording, and overall language fluency. We did not use generative AI tools for any of the research subtasks for which disclosure is required, and the remaining required-disclosure subtasks are not applicable to this work.
All AI-assisted edits were reviewed and verified by two authors to ensure that they preserved the intended meaning and did not introduce unsupported claims or factual errors. The authors made the final decisions regarding the manuscript’s content and wording. We take responsibility for the final content of this work, including all text, claims, and other artifacts produced with the aid of generative AI.

%% file: sections/appendix.tex
\section{{\methodshort} Algorithm Pseudocode}
\label{app:algorithm}

\begin{algorithm}[H]
\small
\setlength{\baselineskip}{1.2\baselineskip}
\caption{{\methodshort} decoding}
\label{alg:eps}
\begin{algorithmic}[1]
\setstretch{1.2}
\Require prompt $y$; generation length $L$; block size $B$; hyperparameters $\tau_{\text{pivot}}, k_{\max}, r, \lambda, \tau$
\State $x \gets \texttt{[MASK]}^L$ \Comment{Initialize response with all masks}
\For{$b = 1, \dots, L/B$} \Comment{Semi-autoregressive block loop}
  \State $\{p_i\} \gets$ forward pass over $[y\Vert x]$ \Comment{Block-opening forward}
  \State $\mathcal{M} \gets$ masked positions in block $b$
  \While{$\mathcal{M} \neq \emptyset$}
    \State $\mathcal{S} \gets \{i \in \mathcal{M} : P_i^{\max} \geq \tau\}$; if $\mathcal{S} = \emptyset$, set $\mathcal{S} \gets \{\arg\max_{i \in \mathcal{M}} P_i^{\max}\}$ \Comment{Standard commit}
    \State Commit $x_i \gets \arg\max_v p_i(v)$ for $i \in \mathcal{S}$; update $\mathcal{M} \gets \mathcal{M} \setminus \mathcal{S}$
    \If{$\mathcal{M} = \emptyset$} \textbf{break} \EndIf
    \State $\mathcal{T}_i \gets \mathrm{top}_{k_{\max}}(p_i)$ and $\mu_i \gets \sum_{v \in \mathcal{T}_i} p_i(v)$ for each $i \in \mathcal{M}$
    \State $\mathcal{M}_{\text{feas}} \gets \{i \in \mathcal{M} : \mu_i \geq \tau_{\text{pivot}}\}$ \Comment{Probability-mass constraint}
    \If{$\mathcal{M}_{\text{feas}} = \emptyset$}
      \State $\{p_i\} \gets$ forward pass over $[y\Vert x]$; \textbf{continue}
    \EndIf
    \State $i^\star \gets \arg\max_{i \in \mathcal{M}_{\text{feas}}} \big[-\sum_{v \in \mathcal{T}_i} p_i(v)\log p_i(v)\big]$ \Comment{Truncated-entropy maximization}
    \State $\mathcal{C} \gets \{c : p_{i^\star}(c) / P_{i^\star}^{\max} \geq r\}$, capped at $k_{\max}$ tokens; add $[\texttt{MASK}]$
    \If{$|\mathcal{C} \setminus \{[\texttt{MASK}]\}| \leq 1$}
      \State $\{p_i\} \gets$ forward pass over $[y\Vert x]$; \textbf{continue}
    \EndIf
    \State $x_{\text{ahead}} \gets [x \;\|\; x_{b}^{(c_1)} \|\; \cdots \;\|\; x_{b}^{(c_{|\mathcal{C}|})}]$; \; $\{p_i^{c}\} \leftarrow$ forward pass over $[y\Vert x_{\text{ahead}}]$ \Comment{Lookahead forward}
    \State $c^\star \gets \arg\max_{c \in \mathcal{C}} \mathrm{score}(c)$ via Eq.~(\ref{eq:score}) \Comment{Scoring and commit}
    \If{$c^\star \neq [\texttt{MASK}]$}
      \State Commit $c^\star$ at $i^\star$; $\mathcal{M} \gets \mathcal{M} \setminus \{i^\star\}$; $\{p_i\} \gets \{p_i^{(c^\star)}\}$ \Comment{Reuse winner's logits}
    \Else
      \State $\{p_i\} \gets \{p_i^{([\texttt{MASK}])}\}$ \Comment{Reuse anchor's logits}
    \EndIf
  \EndWhile
\EndFor
\State \Return $x$
\end{algorithmic}
\end{algorithm}

\section{Proofs for the Lookahead Objective}
\label{app:objective-proof}

\begin{proof}[Proof of Proposition~\ref{prop:entropy-parallelism}]
For $i\in\mathcal M\setminus\{i^\star\}$, define
\begin{equation*}
  q_i=\max_v p_i^c(v),
  \qquad
  \mathcal U_\tau(c)=\{i\in\mathcal M\setminus\{i^\star\}:q_i<\tau\},
  \qquad
  U_\tau(c)=|\mathcal U_\tau(c)|.
\end{equation*}
Fix $i\in\mathcal U_\tau(c)$, so $q_i<\tau<1$. If $q_i\geq 1/2$, let $v_i^\star\in\arg\max_v p_i^c(v)$ and define
$\widetilde p_i^c(v)=p_i^c(v)/(1-q_i)$ for $v\neq v_i^\star$.
The entropy decomposition gives
\begin{align*}
  H(p_i^c)
  &=h(q_i)+(1-q_i)H(\widetilde p_i^c) \\
  &\geq h(q_i).
\end{align*}
If $q_i<1/2$, then $p_i^c(v)\leq q_i$ for every $v$, and hence
\begin{align*}
  H(p_i^c)
  &=\sum_v p_i^c(v)\log\frac{1}{p_i^c(v)} \\
  &\geq\sum_v p_i^c(v)\log\frac{1}{q_i}
   =-\log q_i.
\end{align*}
Since $h(x)$ is non-increasing on $[1/2,1]$ and $h(\tau)\leq\log 2$, for every $i\in\mathcal U_\tau(c)$,
\begin{equation*}
  H(p_i^c)
  \geq
  \begin{cases}
    h(q_i)\geq h(\tau), & q_i\in[1/2,\tau),\\
    -\log q_i>\log 2\geq h(\tau), & q_i<1/2.
  \end{cases}
\end{equation*}
Summing this bound over $\mathcal U_\tau(c)$ yields
\begin{align*}
  n\bar H_c
  =\sum_{i\in\mathcal M\setminus\{i^\star\}}H(p_i^c)
  &\geq\sum_{i\in\mathcal U_\tau(c)}H(p_i^c) \\
  &\geq U_\tau(c)h(\tau),
\end{align*}
and therefore
\begin{equation*}
  U_\tau(c)
  \leq
  \min\!\left\{
    n,
    \left\lfloor\frac{n\bar H_c}{h(\tau)}\right\rfloor
  \right\}.
\end{equation*}
Since $N_\tau(c)=n-U_\tau(c)$,
\begin{align*}
  N_\tau(c)
  &=n-U_\tau(c) \\
  &\geq
  \max\!\left\{
    0,
    n-\left\lfloor\frac{n\bar H_c}{h(\tau)}\right\rfloor
  \right\},
\end{align*}
which is Eq.~(\ref{eq:parallelism-bound}).
\end{proof}

\begin{proof}[Proof of Proposition~\ref{prop:selection-margin}]
Let
\begin{equation*}
  S(c)=-\bar H_c+\lambda\log a_c,
\end{equation*}
Then
\begin{align*}
  \arg\max_{c\in\mathcal C}S(c)
  &=\arg\max_{c\in\mathcal C}
    \{-\bar H_c+\lambda\log a_c\} \\
  &=\arg\min_{c\in\mathcal C}
    \{\bar H_c-\lambda\log a_c\},
\end{align*}
which proves Eq.~(\ref{eq:lagrangian-score}). For any $c,d\in\mathcal C$,
\begin{align*}
  S(c)\geq S(d)
  &\iff
  -\bar H_c+\lambda\log a_c
  \geq
  -\bar H_d+\lambda\log a_d \\
  &\iff
  \bar H_d-\bar H_c
  \geq
  \lambda(\log a_d-\log a_c) \\
  &\iff
  \bar H_d-\bar H_c
  \geq
  \lambda\log\frac{a_d}{a_c},
\end{align*}
which proves Eq.~(\ref{eq:selection-margin}) in both directions.
\end{proof}

\section{Implementation Details}
\label{app:implementation}

\paragraph{Baselines.}
Default uses highest-confidence unmasking with one token per step. Confidence follows Fast-dLLM with threshold $\tau=0.9$. KLASS, EB-Sampler, and WINO follow the settings in their original papers and report the best result from the prescribed sweep ranges: for KLASS, we use confidence threshold $\tau=0.9$ and select the KL threshold $\epsilon_{\mathrm{KL}}$ from $\{0.015, 0.01, 0.005, 0.001\}$; for EB-Sampler, we sweep $\gamma \in \{0.1, 0.01, 0.001\}$ with the confidence-based error proxy; for WINO, we sweep the drafting threshold $\tau_1 \in \{0.5, 0.6, 0.7, 0.8\}$ while fixing the verification threshold $\tau_2=0.9$. For LoPA, we reimplement its core decoding algorithm without the LoPA-Dist distributed inference system or its system-level optimizations, so that all methods run on the same single-device inference stack.

\paragraph{Hardware.}
All experiments are conducted on a node equipped with 4$\times$ NVIDIA A100-SXM4-80GB GPUs. Inference is performed in a data-parallel manner across all four GPUs. Throughput (TPS) is normalized to an equivalent single-GPU basis by dividing the aggregate four-GPU throughput by four.

\section{LLaDA-1.5 Results}
\label{app:llada15}

\begin{table}[H]
\centering
\caption{\textbf{LLaDA-1.5.} Performance and inference speedup comparison across 4 benchmarks.}
\label{tab:main-llada15}
\vspace{4pt}
\small
\setlength{\tabcolsep}{3.5pt}
\begin{tabular}{@{}ll ccccc@{}}
\toprule
\textbf{Benchmark} & \textbf{Method} & \textbf{Acc (\%)} & \textbf{Avg NFE} & \textbf{NFE Speedup} & \textbf{Avg TPS} & \textbf{TPS Speedup} \\
\midrule
\multirow{5}{*}{\shortstack[l]{GSM8K\\(5-shot)}}
& Default & 80.52 & 256.0 & 1.00$\times$ & 3.80 & 1.00$\times$ \\
& Confidence & 81.05 & 74.62 & 3.43$\times$ & 12.96 & 3.41$\times$ \\
& WINO & 80.29 & 57.48 & 4.46$\times$ & 16.64 & 4.38$\times$ \\
& LoPA & 80.21 & 37.61 & 6.81$\times$ & 23.33 & 6.14$\times$ \\
& {\methodshort} & 80.52 & 33.73 & 7.59$\times$ & 25.99 & 6.84$\times$ \\
\midrule
\multirow{5}{*}{\shortstack[l]{HumanEval\\(0-shot)}}
& Default & 42.07 & 256.0 & 1.00$\times$ & 5.03 & 1.00$\times$ \\
& Confidence & 41.46 & 100.92 & 2.54$\times$ & 12.90 & 2.56$\times$ \\
& WINO & 42.68 & 84.76 & 3.02$\times$ & 14.08 & 2.80$\times$ \\
& LoPA & 40.24 & 47.46 & 5.39$\times$ & 24.30 & 4.83$\times$ \\
& {\methodshort} & 41.46 & 42.31 & 6.05$\times$ & 27.21 & 5.41$\times$ \\
\midrule
\multirow{5}{*}{\shortstack[l]{MATH500\\(4-shot)}}
& Default & 39.80 & 256.0 & 1.00$\times$ & 5.32 & 1.00$\times$ \\
& Confidence & 39.60 & 95.11 & 2.69$\times$ & 14.26 & 2.68$\times$ \\
& WINO & 40.20 & 75.30 & 3.40$\times$ & 17.83 & 3.35$\times$ \\
& LoPA & 38.20 & 57.70 & 4.44$\times$ & 22.80 & 4.29$\times$ \\
& {\methodshort} & 39.60 & 49.37 & 5.19$\times$ & 25.91 & 4.87$\times$ \\
\midrule
\multirow{5}{*}{\shortstack[l]{MBPP\\(3-shot)}}
& Default & 38.40 & 256.0 & 1.00$\times$ & 1.46 & 1.00$\times$ \\
& Confidence & 38.40 & 42.68 & 6.00$\times$ & 8.37 & 5.73$\times$ \\
& WINO & 39.00 & 46.48 & 5.51$\times$ & 7.91 & 5.42$\times$ \\
& LoPA & 40.20 & 24.66 & 10.38$\times$ & 14.38 & 9.85$\times$ \\
& {\methodshort} & 44.20 & 23.50 & 10.89$\times$ & 14.40 & 9.86$\times$ \\
\bottomrule
\end{tabular}
\end{table}

\section{Limitations}
\label{app:limitations}

\paragraph{Limitations.}
While {\methodshort} demonstrates strong acceleration gains while largely preserving generation quality, it still has several limitations. 
% First, although the plausibility weight $\lambda$ is not highly sensitive within a broad validated plateau (\S\ref{sec:exp-aw-sweep}), it still requires per-task selection within that range and therefore does not constitute zero-tuning in the strictest sense. 
First, our characterization of the {\effectname} is empirical and qualitative rather than theoretically derived. In particular, the motivating analysis in \S\ref{sec:intro} is conducted on GSM8K and is intended to illustrate the phenomenon, not to establish that the precise location of the cascade peak quantitatively generalizes across models or tasks. 
Second, when combined with prefix cache acceleration (\S\ref{sec:exp-cache}), accuracy degradation may arise from the cache approximation itself. This effect is not specific to {\methodshort} and is also observed for other parallel decoding methods, but mitigating it remains outside the scope of this work.